%% file: neurips_2026.tex
\documentclass{article}

\usepackage[dblblindworkshop, final]{neurips_2026}
\workshoptitle{Decision-Making from Offline Datasets to Online Adaptation: Black-Box Optimization to Reinforcement Learning}

\usepackage[utf8]{inputenc} 
\usepackage[T1]{fontenc}    
\usepackage{hyperref}       
\usepackage{url}            
\usepackage{booktabs}       
\usepackage{amsfonts}       
\usepackage{nicefrac}       
\usepackage{microtype}      
\usepackage{xcolor}         
\usepackage{amsmath}
\usepackage{amssymb}
\usepackage{mathtools}
\usepackage{amsthm}
\newtheorem{theorem}{Theorem}
\newtheorem{proposition}{Proposition}

\usepackage{multirow}
\title{Second-Order Actor-Critic Methods for \\ Discounted MDPs via Policy Hessian Decomposition}

\author{%
Sanjeev Manivannan \\
Indian Institute of Technology Madras \\
Chennai, India \\
\And
Shuban Venkataraman \\
Indian Institute of Technology Madras \\
Chennai, India \\
}

\begin{document}

\maketitle

\begin{abstract}
We address the discounted reward setting in reinforcement learning (RL). To mitigate the value approximation challenges in policy gradient methods, actor-critic approaches have been developed and are known to converge to stationary points under suitable assumptions. However, these methods rely on first-order updates. In contrast, second-order optimization provides principled curvature-aware updates that are proven to accelerate convergence, but its application in RL is limited by the computational complexity of Hessian estimation. In this work, we analyze second-order approximations for the actor update that leverage the full curvature information of the objective as much as possible. A stable approximation requires treating the action-value function as locally constant with respect to policy parameters, which does not generally hold in policy gradient methods. We show that this approximation becomes well-justified under a two-timescale actor-critic framework, where the critic evolves on a faster timescale and can be treated as quasi-stationary during actor updates. Building on this insight, we formulate a second-order actor-critic method for the discounted reward setting that leverages Hessian-vector product (HVP) computations, resulting in a computationally efficient and stable second-order update.
\end{abstract}

\section{Introduction}

Markov Decision Processes (MDPs) provide a framework for sequential decision making under uncertainty, where an agent interacts with an environment defined by states, actions, transition dynamics, and rewards. When the transition dynamics are known, dynamic programming methods such as value iteration \cite{bellman1957dynamic} and policy iteration \cite{howard1960dynamic} compute optimal policies via recursive Bellman equations over value functions. These formulations have been successfully applied to a wide range of control problems, including traffic signal control~\cite{wiering2000multiagent}, robotics~\cite{kober2013reinforcement}, autonomous navigation~\cite{paden2016survey}, and nonlinear dynamical systems~\cite{bertsekas1995neurodynamic}. However, in real world scenarios, the transition dynamics are unknown and therefore one has to resort to RL methods that learn optimal behavior through sampled interactions using approaches such as Monte-Carlo methods \cite{sutton2018reinforcement}, (TD($\lambda$)) learning \cite{sutton1988learning}, and Q-learning \cite{watkins1992q}.

Even though these methods converge to stationary points under suitable conditions, they rely on tabular representations of state or state-action value functions which scale with the size of the state and action spaces. Consequently, in large-scale settings, these representations become intractable due to the curse of dimensionality \cite{bellman1957dynamic}. This has led to the development of policy parameterization \cite{sutton2000policygradient} and modern policy search methods, where the policy is represented as a parametric function $\pi_{\theta}(a\mid s)$, often defined over feature representations $\phi(s)$ or learned function approximators parameterized as $\theta$. There are two major challenges that are being addressed in policy search methods.

The \textit{first challenge} in policy search methods is value estimation. Standard policy gradient approaches often rely on Monte Carlo returns \cite{sutton2018reinforcement}, which provide high variance estimates of the return \cite{sutton1988learning} using complete trajectories and therefore require waiting until the end of an episode before performing updates based on stochastic approximation \cite{borkar1997stochastic}. To address this limitation, actor-critic methods \cite{konda2000actorcritic} were introduced where a parametric value function (critic) is learned alongside the policy to provide relatively low variance estimates of the return in every state. In particular, two-timescale actor-critic approaches \cite{bhatnagar2009actorcritic}, where the critic evolves on a faster timescale than the actor, provide a stable framework in which the value function can be treated as approximately stationary during policy updates.

The \textit{second challenge} lies in incorporating curvature information during optimization. Standard policy gradient methods such as REINFORCE~\cite{sutton2000policygradient} rely on first-order updates that are not affine-invariant, i.e., their performance depends on the parameterization of the policy and does not account for the underlying geometry of the objective. As a result, these methods ignore curvature information, which can lead to slow convergence. In contrast, second-order optimization offers principled curvature-aware updates that are invariant to affine transformations of the parameter space and can significantly improve convergence efficiency~\cite{kakade2001natural, schulman2015trpo}. However, its application in reinforcement learning is limited due to the computational complexity ($O(d^{3})$) of Hessian estimation as well as the instability induced by the high-variance interaction terms in the policy Hessian decomposition~\cite{prashanth2018zeroth}. This motivates the need for structured approximations \cite{furston2016approximate} of the Hessian that mitigate the effect of the high variance cross terms and the curvature components that impede convergence of the objective, while preserving useful curvature information for efficient optimization.

\textbf{Our work:} Prior work typically addresses these challenges independently. In this work, we establish a connection between them:

$\bullet$ \textbf{Two-timescale insight:} Second-order approximations are generally unstable due to high-variance curvature components and can lead to loss of useful curvature information. Stabilizing these approximations requires: (i) low-variance state-action value estimates, which are supported by the rapidly evolving critic in the two-timescale actor-critic framework, and (ii) treating the action-value function as locally constant with respect to the policy parameters during actor updates. Although this assumption does not generally hold in policy gradient methods, it becomes more valid in the two-timescale actor-critic framework \cite{bhatnagar2009actorcritic}, where the critic evolves on a faster timescale than the actor and can be treated as quasi-stationary \cite{konda2000actorcritic}.
    
$\bullet$ \textbf{Efficient Second-order updates:} Building on this insight, we analyze second-order approximations \cite{furston2016approximate} that efficiently capture the curvature structure of the objective, and show that they can be implemented in the two-timescale actor-critic setting using Hessian-vector products. This results in improved sample efficiency and stable second-order updates in the discounted MDP setting.

In this paper, we further demonstrate our claims across a range of benchmark environments, ranging from discrete control tasks such as CartPole and LunarLander \cite{brockman2016openai}, to continuous control domains from MuJoCo environments including Reacher and Humanoid \cite{todorov2012mujoco}.

\section{Mathematical Preliminaries}

In this section, we introduce Markov decision processes and the standard terminology used throughout the paper. We further discuss policy search methods \cite{sutton2000policygradient}, policy Hessian decomposition \cite{prashanth2018zeroth}, and actor-critic methods \cite{bhatnagar2009actorcritic}.

\subsection{Markov Decision Processes (MDPs)}

In a Markov decision process an agent (or controller) sequentially interacts with an environment by selecting actions based on the current state of the environment. The system then transitions to a new state, often in a stochastic manner, and the agent receives a scalar reward that depends on the selected action and the resulting state. In our setting, the objective is to maximize the expected total discounted reward. Formally an MDP is defined using a tuple $(S,A,P,R,\gamma,p_i)$, where $S$ and $A$ denote the action and state spaces respectively, $R(s,a)$ represents the reward function,  $P(s' \mid s,a)$ defines the transition dynamics, $\gamma \in [0,1)$ is the discount factor and $p_i(s_o)$ represents the initial state distribution. 

We describe the evolution of an MDP as follows. At time step $t$, the agent observes state $s_t \sim p_t$, selects an action $a_t \sim \pi(\cdot \mid s_t)$, transitions to the next state $s_{t+1} \sim P(\cdot \mid s_t, a_t)$, and receives a scalar reward determined by $R$. We consider a finite horizon discounted setting. We represent the states and actions until the termination at $H$ using a trajectory $\tau = (s_0, a_0, s_1, a_1, \dots, s_{H-1}, a_{H-1}, s_H)$ where $H$ is the trajectory or episode length. The probability of trajectory $\tau$ following policy $\pi$ is given by:
\begin{equation}
p(\tau \mid \pi) = \Big(\prod_{t=0}^{H-1} P(s_{t+1} \mid s_t, a_t) \pi(a_t \mid s_t) \Big) p_i(s_0)
\end{equation}
We denote the total discounted reward of a trajectory $\tau$ using $D(\tau) = \sum_{t=0}^{H-1} \gamma^{t} R(s_t, a_t)$ where $\gamma \in [0, 1)$. Our objective is to maximize the expected cumulative discounted reward under a policy $\pi$:
\begin{equation}
J(\pi) = \mathbb{E}_{\tau \sim p(\tau \mid \pi)}\big[D(\tau)\big] = \mathbb{E}_{s_t, a_t \sim p_t}\big[\sum_{t=0}^{H-1}\gamma^{t} R(s_t, a_t)\big]  
\end{equation}
The objective $J(\pi)$ represents the expected return starting from the initial state distribution. To characterize returns at the state and state--action level, we define the value functions as follows. The state-value function is defined as
\begin{equation}
V(s ;\pi) = \mathbb{E}_{\tau \sim p(\tau \mid s_0 = s, \pi)}\left[ \sum_{t=0}^{H-1} \gamma^t R(s_t, a_t)\right],
\end{equation}
which represents expected return starting from state $s$. The state--action value function is defined as
\begin{equation}
Q(s,a ;\pi) = \mathbb{E}_{\tau \sim p(\tau \mid s_0 = s, a_0 = a, \pi)}\left[ \sum_{t=0}^{H-1} \gamma^t R(s_t, a_t) \right],
\end{equation}
which represents the expected return starting from state $s$ and taking action $a$. The advantage function measures the relative value of an action compared to the average value of a state, and is defined as
\begin{equation}
A(s,a; \pi) = Q(s,a; \pi) - V(s; \pi),
\label{eq:generalized-advantage-estimate}
\end{equation}
where $V(s;\pi) = \mathbb{E}_{a \sim \pi(\cdot \mid s)}\big[Q(s,a;\pi)\big]$. In practice, we estimate advantages using TD(0)-based advantage estimate, which provides bias--variance trade-off for policy gradient updates.
\subsection{Policy Search Methods}
In large-scale settings, classical MDP formulations suffer from the curse of dimensionality due to exponential growth of the state--action space. To address this, modern policy search methods assume policy parameterization by a vector $\theta \in \mathbb{R}^d$. In this paper, we use the notation $\pi_\theta$ as a shorthand for the distribution $\pi(a_t \mid s_t; \theta)$,   $p(\tau; \theta) \equiv p(\tau; \pi_\theta)$ for the trajectory distribution and $J(\theta) \equiv J(\pi)$ for the expected cost. For simplicity, we assume terminal reward is $0$. Our aim is to find a parameter
\begin{equation}
\theta^* \in \arg\max_{\theta \in \mathbb{R}^d} J(\theta).
\end{equation}
We use the discounted occupancy measure \cite{furston2016approximate} induced by a policy $\pi_\theta$ as
\begin{equation}
\rho_\gamma(s,a;\theta) = \sum_{t=0}^{H-1} \gamma^t \Pr(s_t = s, a_t = a \mid \pi_\theta).
\end{equation}
where, $\sum_{s,a \in \mathcal{S} \times \mathcal{A}} \rho_{\gamma}(s,a;\theta) = (1-\gamma)^{-1}$ for an infinite horizon $H \to \infty$. Additionally, the total discounted reward function can be written as $J(\theta) = \sum_{s,a \in \mathcal{S} \times \mathcal{A}} \rho_\gamma(s,a;\theta) R(s,a)$. The gradient $\nabla J(\theta)$ can be written as
\begin{equation}
\nabla J(\theta)
= \sum_{t=0}^{H-1} \sum_{\tau} \gamma^{t-1} R(s_t, a_t)\, \nabla p(\tau; \theta).
\end{equation}
One of the most widely used approaches in policy search follow the principles of stochastic approximation \cite{borkar1997stochastic} commonly referred to as policy gradient methods. A canonical instance is the REINFORCE algorithm \cite{sutton2000policygradient}. These methods rely on first-order updates given by:
\begin{equation}
\theta_{k+1} = \theta_{k} + \alpha \nabla J(\theta_{k})
\label{eq:policy-gradient}
\end{equation}
However, first order updates generally ignore curvature information and are often poorly scaled, leading to update directions that are sensitive to the policy parameterization. This motivates the use of preconditioning, where the gradient is rescaled before updating to account for the local geometry of the objective. In this context, stochastic Newton methods \cite{prashanth2018zeroth} introduce a preconditioning matrix $M(\theta)$ in \eqref{eq:policy-gradient} as given below:
\begin{equation}
\theta_{k+1} = \theta_{k} + \alpha M(\theta_k) \nabla J(\theta_k)
\end{equation}
A widely used choice of preconditioning matrix in policy optimization is given by the Natural Policy Gradient (NPG), where the preconditioning metric $M(\theta)$ is taken as the inverse of the Fisher Information matrix of the policy distribution. Specifically, the Fisher matrix is given by 
\begin{equation}
G(\theta) = \mathbb{E}_{s_t, a_t \sim p_t} \left[ \nabla_\theta \log \pi_\theta(a_t \mid s_t) \, \nabla_\theta \log \pi_\theta(a_t \mid s_t)^\top \right],
\end{equation}
and the corresponding preconditioner is $M(\theta) = G(\theta)^{-1}$. The use of the Fisher Information Matrix can be interpreted as imposing a local norm on the parameter space, derived from a second-order approximation of the KL-divergence between trajectory distributions \cite{coolen2005theory, furston2016approximate}. While the Fisher information matrix captures curvature arising from the policy distribution, it does not represent the true Hessian of the objective $J(\theta)$  and therefore misses important second-order interactions in the expected return. To fully characterize the curvature of the objective, we consider the policy Hessian theorem, obtained by taking the second-order derivative of the REINFORCE objective \cite{maniyar2024crpn, furston2016approximate}.

\subsection{Policy Hessian Decomposition}
\label{sec:policy-hessian}

To fully characterize the curvature of the policy optimization objective, we consider the policy Hessian theorem \cite{baxter2001infinite, kakade2001natural, maniyar2024crpn, furston2016approximate}, obtained by differentiating the policy gradient objective a second time.

\begin{theorem}[Policy Hessian Theorem]
Consider a Markov decision process with objective $J(\theta)$ and induced trajectory distribution $p(\tau; \theta)$. For any given parameter vector $\theta \in \mathbb{R}^{d}$, the Hessian of the objective can be decomposed as:
\begin{equation}
\mathcal{H}(\theta) = \nabla^2 J(\theta) = \mathcal{H}_1(\theta) + \mathcal{H}_2(\theta) + \mathcal{H}_{12}(\theta) + \mathcal{H}_{12}^\top(\theta).
\label{eq:hessian-decomposition}
\end{equation}
\end{theorem}
in which the components $\mathcal{H}_1, \mathcal{H}_2$ and $\mathcal{H}_{12}$ are given by:
\begin{align}
\mathcal{H}_1(\theta) &= \mathbb{E}_{\rho_{\gamma}} \left[
Q(s_t, a_t ; \theta)\, \nabla_\theta \log \pi_\theta(a_t \mid s_t)\, \nabla_\theta \log \pi_\theta(a_t \mid s_t)^\top
\right], \\
\mathcal{H}_2(\theta) &= \mathbb{E}_{\rho_{\gamma}} \left[
Q(s_t, a_t; \theta)\, \nabla_\theta^2 \log \pi_\theta(a_t \mid s_t)
\right], \\
\mathcal{H}_{12}(\theta) &= \mathbb{E}_{\rho_{\gamma}} \left[
\nabla_\theta \log \pi_\theta(a_t \mid s_t)\, \nabla_\theta^\top Q(s_t, a_t; \theta)
\right].
\end{align}
where $\mathbb{E}_{\rho_\gamma}[\cdot] = \sum_{s,a \in \mathcal{S} 
\times \mathcal{A}} \rho_{\gamma}(s,a;\theta)(\cdot)$. Here, $\mathcal{H}_1$ corresponds to the outer-product term, $\mathcal{H}_2$ captures the intrinsic curvature of log-policy, and $\mathcal{H}_{12}$ represents interaction terms involving gradient of state-action value function. Each term is weighted by the state--action value $Q(s_t, a_t;\theta)$.

\subsection{Actor--Critic Methods}

In the policy gradient methods especially in a monte carlo setup, the gradients are estimated using sampled trajectories, which leads to high variance due to the dependence on a long-horizon. A central challenge in policy optimization is therefore the estimation of the action--value function $Q^\pi(s,a)$.

Actor--critic methods address this issue by introducing a separate function parameterization, known as the \emph{critic}, to approximate the action--value function. The policy, referred to as the \emph{actor}, is parameterized by $\pi_\theta(a \mid s)$ and is updated using gradients of the form
\begin{equation}
\nabla J(\theta) = \mathbb{E}\left[ Q^{c}(s,a)\, \nabla_\theta \log \pi_\theta(a \mid s) \right],
\end{equation}
where $Q^{c}(s,a)$ denotes the critic’s value estimate. By replacing Monte Carlo returns with learned value estimates, actor--critic methods significantly reduce variance while retaining the direction.
\\ \\
\noindent\textbf{Two-timescale framework:} A common formulation employs a two-timescale stochastic approximation framework \cite{konda2000actorcritic, borkar1997stochastic}, in which the critic parameters are updated on a faster timescale than the actor. Intuitively, the critic rapidly adapts to the current policy, providing accurate value approximation that guide the slower policy updates. This separation of timescales enables stable learning and forms the foundation for many modern policy optimization algorithms.

\section{Method}
\subsection{Policy Hessian and HVP formulation}
While the second-order updates incorporate curvature information and can accelerate convergence, explicitly computing the Hessian incurs $\mathcal{O}(d^{3})$ complexity. This motivates the use of Hessian--Vector Products \cite{maniyar2024crpn} which enable efficient computation of curvature without forming the full hessian. In the following, we reformulate the policy Hessian decomposition in terms of HVPs for efficient implementation. For simplicity, we denote the state-action value function $Q(s,a;\theta)$ by $Q$, and the policy $\pi_{\theta}(a_t \mid s_t)$ by $\pi_{\theta}$. We further denote the discounted state--action occupancy measure $\rho_{\gamma}(s,a;\theta)$ using $\rho_{\gamma}$. 
\\ \\
\noindent\textbf{Hessian--Vector Product (HVP) formulation:}
To obtain a second-order update, we consider a Taylor approximation of the objective around the current policy parameter $\theta$:
\begin{equation}
J(\theta+\Delta)
=
J(\theta)
+
\nabla J(\theta)^\top \Delta
+
\frac{1}{2}\Delta^\top \nabla^2 J(\theta)\Delta
+
\mathcal{O}(\|\Delta\|^3).
\end{equation}
Let $\tilde{g}$ be the stochastic estimate of policy gradient $g$ derived as given below:
\begin{equation}
  g = \nabla J(\theta) = \sum_{s,a} \rho_\gamma(s,a;\theta)\, Q^\pi(s,a;\theta)\,\nabla \log \pi_\theta(a|s) = \mathbb{E}_{\rho_\gamma} \left[ Q^\pi \nabla \log \pi \right].
\label{eq:policygradient}
\end{equation}
Differentiating \eqref{eq:policygradient} with respect to $\theta$ yields the policy hessian $\mathcal{H} = \nabla^2 J(\theta)$ which admits the decomposition in \eqref{eq:hessian-decomposition} as described in Section \ref{sec:policy-hessian}. Let $\tilde{\mathcal{H}}$ denote stochastic estimates of $\mathcal{H}$ derived from \eqref{eq:hessian-decomposition} and let $\lambda I$ be a ridge term (damping) to stabilize the curvature estimate, where $-\infty < \lambda < 0$ to ensure negative-definiteness for maximization. The update direction is obtained by approximately maximizing the below local quadratic model (because $J(\theta)$ doesn't have any effect on $\Delta$):
\begin{equation}
\Delta^\star
=
\arg\max_{\Delta}
\left[
\tilde{g}^\top \Delta
+
\frac{1}{2}
\Delta^\top
(\tilde{\mathcal{H}}+\lambda I)
\Delta
\right].
\end{equation}
Due to the computational complexity, we explicitly construct $\tilde{\mathcal{H}}$ by computing its action on the current displacement $\Delta$ using Hessian--vector products \cite{maniyar2024crpn} given by 
$\tilde{\mathcal{H}}\Delta = \nabla_\theta \left( \tilde{g}^\top \Delta \right)$.Thus, the inner-loop update for $\Delta$ is given by
\begin{equation}
\Delta_{k+1}
=
\Delta_k
+
\alpha
\left[
\tilde{g}_k
+
(\tilde{\mathcal{H}}_k+\lambda I)\Delta_k
\right].
\label{eq:delta-inner-loop}
\end{equation}
After the inner loop, the policy parameters are updated as $\theta \leftarrow \theta + \Delta$, yielding a stable second-order update that incorporates curvature information. Using the Hessian decomposition, we express each component in vector form. In particular, the outer-product (Fisher-like) term $\mathcal{H}_1$ acting on $\Delta$ is given by 
\begin{equation}
\mathcal{H}_1\cdot \Delta
=
\mathbb{E}_{\rho_{\gamma}}
\left[
Q\,
\nabla \log \pi_\theta\,
\big( \nabla \log \pi_\theta^\top \Delta \big)
\right]
\end{equation}
However, for the intrinsic term $\mathcal{H}_2 = \mathbb{E}_{\rho_{\gamma}}
\left[
Q\,
\nabla_\theta^{2} \log \pi_\theta\,\right]$ a direct Hessian-Vector product form computation is subtle because $Q = Q(s,a;\theta)$ also depends on the policy parameters. In an other way, we can represent the policy hessian decomposition in the below form combining the $\mathcal{H}_2$ and $\mathcal{H}_{12}$ into one term:
\begin{equation}
\begin{aligned}
\nabla^2 J(\theta) = \mathbb{E}_{\rho_\gamma} \left[
Q\,\nabla \log \pi_{\theta}\, \nabla \log \pi_{\theta}^\top
+ \nabla \left( Q\,\nabla \log \pi_{\theta} \right)
\right]
\end{aligned}
\label{eq:hessian-combined-decomposition}
\end{equation}
By the product rule, the second term in Eq \eqref{eq:hessian-combined-decomposition} would be 
\begin{equation}
(\mathcal{H}_2 + \mathcal{H}_{12})\cdot\Delta = \nabla (Q\,\nabla \log \pi_{\theta}^\top\Delta),
\label{eq:H1H2-HVPform}
\end{equation}
rather than $\mathcal{H}_2\cdot\Delta$ alone. This coupling makes it difficult to separately estimate the intrinsic curvature term without incurring the interaction term $\mathcal{H}_{12}$. The motivation for isolating and approximating $\mathcal{H}_2$ is discussed in the next section, where we analyze the limitations of each component in the Hessian decomposition. 

\subsection{Limitations of Full Hessian Curvature}
\label{subsec:limitations}
Despite $O(d^{3})$ complexity, the full Hessian also exhibits structural issues that hinder performance: 
\begin{itemize}
    \item \textbf{High-variance interaction term:} The cross term $\mathcal{H}_{12}$ captures the sensitivity of state-action value function to policy parameters. Since $Q$ depends on full trajectories, its gradient $\nabla Q$ introduces significant variance, making the curvature estimate unstable and difficult to compute reliably. 
    \item \textbf{Positive definite bias $\mathcal{H}_1$:} The outer-product term $\mathcal{H}_1$ contributes a positive semi-definite component (under non-negative weighting), which introduces convex structure into the Hessian. This conflicts with the concave structure required for stable maximization, leading to distorted curvature estimates and suboptimal directions. 
\end{itemize} 
Together, these effects make the full Hessian indefinite and poorly conditioned, making direct second order approximation unstable in practice. These challenges motivate the need for structured approximations \cite{furston2016approximate} which we address in the following sections. These structural issues of second-order curvature in policy optimization have been noted in prior work.

\subsection{Two-Timescale Approximation and Actor-Critic Structure}
\label{subsec:TTAC}
To stabilize the Hessian, we first address the high-variance interaction term $\mathcal{H}_{12}$:
\begin{itemize}
    \item \textbf{Variance limitation:} As mentioned in Sec \ref{subsec:limitations}, in policy optimization, the state-action value function depends on policy parameters, i.e., $\nabla_\theta Q \neq 0$, making $\mathcal{H}_{12}$ difficult to estimate.
    \item \textbf{HVP coupling:} In addition to its high variance, $\mathcal{H}_{12}$ also complicates the computation of HVP. Differentiating the product $Q \nabla \log \pi$ shown in \eqref{eq:H1H2-HVPform} introduces the cross term $(\nabla Q)(\nabla \log \pi)  = \mathcal{H}_{12}$. Since $\nabla Q$ depends on long-horizon trajectory distributions. Therefore, computing $(\mathcal{H}_2 + \mathcal{H}_{12})\cdot \Delta$ without HVPs becomes computationally expensive and statistically unstable, preventing a clean isolation of the intrinsic curvature term $\mathcal{H}_2$.
    \item \textbf{Two-timescale insight:}
Under the two-timescale actor--critic framework
\cite{konda2000actorcritic,borkar1997stochastic},
the action-value function is estimated using a separate
critic network parameterized by $\omega$. Since the critic
evolves on a faster timescale than the actor, it rapidly
tracks the current policy while the actor changes
gradually. Consequently, during actor updates, the critic
may be treated as approximately stationary with respect
to the actor parameters, i.e., locally quasi-stationary
or nearly converged \cite{konda2000actorcritic}.
This suggests that the sensitivity of the critic estimate
with respect to the actor parameters is small, implying $
\nabla_\theta Q^c(s,a;\omega) \approx 0$. As a result, the interaction term $\mathcal{H}_{12}$
appearing in the policy Hessian decomposition is expected
to be negligible. 
\end{itemize} 

To formalize this observation, we make
the following \textit{assumptions}. Let $\omega$ denote the critic
parameters and let $\|\cdot\|$ denote the Euclidean norm
for vectors and the induced operator norm for matrices.

\noindent \textbf{(A1) Smooth Objective.}
The objective function \(J:\mathbb{R}^{d}\rightarrow\mathbb{R}\) is twice continuously differentiable with \(L\)-Lipschitz continuous gradients and \(\rho\)-Lipschitz continuous Hessian. That is, there exist constants \(L,\rho>0\) such that, for all \(\theta,\theta'\in\mathbb{R}^{d}\),
\begin{equation}
\|\nabla J(\theta)-\nabla J(\theta')\|
\le
L\|\theta-\theta'\|,
\end{equation}
and
\begin{equation}
\|\nabla^2J(\theta)-\nabla^2J(\theta')\|
\le
\rho\|\theta-\theta'\|.
\end{equation}

\noindent\textbf{(A2) Parameterization regularity.}
The policy $\pi_\theta(a|s)$ is twice continuously
differentiable with respect to $\theta$. Consequently,
for every parameter vector $\theta$ and every state--action
pair $(s,a)$, there exist constants
$0<C_g,C_H<\infty$ such that
\begin{equation}
\|\nabla_\theta \log \pi_\theta(a|s)\|
\leq C_g,
\qquad
\|\nabla_\theta^2 \log \pi_\theta(a|s)\|
\leq C_H,
\quad
\forall (s,a)\in\mathcal{S}\times\mathcal{A}.
\end{equation}
\noindent\textbf{(A3) Smooth critic.}
The critic $Q^c(s,a;\omega)$ is differentiable with
respect to $\omega$ and its parameter gradient is
uniformly bounded:
\begin{equation}
\|\nabla_\omega Q^c(s,a;\omega)\|
\leq C_c,
\quad
\forall (s,a)\in\mathcal{S}\times\mathcal{A}.
\end{equation}
\noindent\textbf{(A4) Critic tracking.}
For every actor parameter $\theta$, let
$\omega^\star(\theta)$ denote the optimal critic
parameter corresponding to policy $\pi_\theta$.
Under two-timescale actor--critic updates, the critic
tracks this optimum with bounded error:
\begin{equation}
\|
\omega-\omega^\star(\theta)
\|
\leq
\epsilon_c,
\end{equation}
where $\epsilon_c$ is small whenever the critic
timescale is sufficiently faster than the actor
timescale.

\noindent\textbf{(A5) Slow actor variation.}
During actor updates, the critic parameter $\omega$
is treated as approximately fixed with respect to
the actor parameter $\theta$:
\begin{equation}
\frac{\partial \omega}{\partial \theta}
\approx 0.
\end{equation}
Based on Assumptions A1--A5, the following proposition provides a formal justification for the two-timescale insight.
\begin{proposition}
\label{prop:h12_bound}
Under Assumptions A1--A5, the actor--critic interaction term
\begin{equation}
\mathcal{H}_{12}^{AC}(\theta)
=
\mathbb{E}_{\rho_\gamma}
\left[
\nabla_\theta \log \pi_\theta(a|s)
\nabla_\theta Q_c(s,a;\omega)^\top
\right]
\label{eq:actor-critic-interaction-term}
\end{equation}
satisfies
\begin{equation}
\|H_{12}^{AC}(\theta)\| \leq C_g C_c \left\|\frac{\partial \omega}{\partial \theta}\right\|.
\end{equation}
In particular, under the two-timescale quasi-stationary approximation
$\partial \omega / \partial \theta \approx 0$, we obtain $
H_{12}^{AC}(\theta) \approx 0$. Consequently,
\begin{equation}
\mathcal{H}
\approx
\mathcal{H}_1
+
\mathcal{H}_2.
\end{equation}
\end{proposition}
Therefore the above proposition justifies the two-timescale insight, implying that the interaction term becomes
negligible, yielding the interaction-free Hessian
approximation $\widetilde{\mathcal H} = \mathcal H_1+\mathcal H_2$.
This approximation enables a stable and efficient
Hessian--vector product computation without
incurring the high variance associated with
$\mathcal H_{12}$. A detailed proof of Proposition~\ref{prop:h12_bound}
is provided in Appendix~\ref{app:two-timescale-insight}.

\subsection{Second-Order Approximations}
As mentioned in Sec \ref{subsec:TTAC}. the full Hessian contains a high variance interaction term $\mathcal{H}_{12}$, which is difficult to estimate reliably and complicates Hessian vector product (HVP) computation. Therefore, we approximate the Hessian using the \textit{two-timescale actor-critic method}, resulting in a interaction free Hessian estimate given by $\tilde{\mathcal{H}} = \mathcal{H}_1 + \mathcal{H}_2$.
We now consider standard approximations of the $\tilde{\mathcal{H}}$:

\subsubsection{ACGN1: Advantage-based Gauss--Newton approximation.}
\label{subsec:acgn1}
We know that $Q(s,a;\theta) =  V(s;\theta) + A(s,a;\theta)$ from Eq \eqref{eq:generalized-advantage-estimate}. For simplicity we represent $Q(s,a;\theta), V(s;\theta)$ and $A(s,a;\theta)$ as $Q, V$ and $A$ respectively. Let \(\mathcal{A}_1,\mathcal{A}_2\) and \(\mathcal{V}_1,\mathcal{V}_2\) denote the curvature components obtained from the Hessian decomposition in Theorem~1 by replacing the  \(Q\) with \(A\) and \(V\), respectively. Explicitly,
\begin{equation}
\begin{aligned}
\mathcal{V}_1(\theta)
&=
\mathbb{E}_{\rho_\gamma}
\left[
V(s_t;\theta)\,
\nabla_\theta \log \pi_\theta(a_t|s_t)\,
\nabla_\theta \log \pi_\theta(a_t|s_t)^\top
\right], \\
\mathcal{V}_2(\theta)
&=
\mathbb{E}_{\rho_\gamma}
\left[
V(s_t;\theta)\,
\nabla_\theta^2 \log \pi_\theta(a_t|s_t)
\right], \\
\mathcal{A}_1(\theta)
&=
\mathbb{E}_{\rho_\gamma}
\left[
A(s_t,a_t;\theta)\,
\nabla_\theta \log \pi_\theta(a_t|s_t)\,
\nabla_\theta \log \pi_\theta(a_t|s_t)^\top
\right], \\
\mathcal{A}_2(\theta)
&=
\mathbb{E}_{\rho_\gamma}
\left[
A(s_t,a_t;\theta)\,
\nabla_\theta^2 \log \pi_\theta(a_t|s_t)
\right].
\end{aligned}
\label{eq:value-advantage-hessian-approximations}
\end{equation}
Applying linearity of expectation to Eq \eqref{eq:value-advantage-hessian-approximations},
\begin{equation}
\mathcal{H}_1=\mathcal{V}_1+\mathcal{A}_1,
\qquad
\mathcal{H}_2=\mathcal{V}_2+\mathcal{A}_2.
\end{equation}
Under standard Fisher regularity conditions $\mathcal{V}_1 = -\mathcal{V}_2$ \cite{furston2016approximate} which implies the curvature terms defined using state-action function $Q$ can be expressed in terms of the advantage function A. Hence, the interaction-free Hessian admits the advantage-based representation leading to the advantage-based Gauss--Newton approximation
\begin{equation}
\tilde{\mathcal{H}}_{\text{ACGN1}} = \mathcal{H}_1+\mathcal{H}_2
=
\mathcal{A}_1+\mathcal{A}_2.
\label{eq:gauss-newton1-approximation}
\end{equation}
which incorporates both outer-product ($\mathcal{A}_1$) and intrinsic curvature ($\mathcal{A}_2$) components, while benefiting from variance reduction through advantage estimation. Therefore based on this approximation, the corresponding Hessian-Vector product would be the following:
\begin{equation}
    \tilde{\mathcal{H}}\cdot\Delta = \tilde{\mathcal{H}}_{\text{ACGN1}}\cdot\Delta = (\mathcal{A}_1 + \mathcal{A}_2)\cdot\Delta
\end{equation}
Additionally using advantage-based Hessian estimates reduces the scale of variance in the curvature approximation and the Hessian-vector products. With appropriately chosen small step sizes for the actor updates, this residual variance is less likely to induce instability and instead yields more reliable directions. We emphasize that, as mentioned in \cite{furston2016approximate}, this condition doers not imply negative-definiteness of the resulting curvature matrix $\mathcal{H}_{\text{ACGN1}}$ for reward maximization or guaranteed alignment with the true Newton direction; rather, it is an empirical observation that induced updates are often well behaved in the vicinity of the local optima.

\subsubsection{ACGN2: Intrinsic curvature approximation}
\label{subsec:acgn2}
An alternative approximation to address the structural bias introduced into the curvature by the outer-product term $\mathcal{H}_1$ is to discard $\mathcal{H}_1$ entirely and retain only the intrinsic curvature component $\mathcal{H}_2$. This yields the following intrinsic curvature Gauss-Newton approximation:
\begin{equation}
\tilde{\mathcal{H}}\cdot\Delta =
\tilde{\mathcal{H}}_{\text{ACGN2}}\cdot\Delta = \mathcal{H}_2\cdot\Delta
\end{equation}
Under log-concave policy parameterizations (e.g., softmax or Gaussian policies) \cite{furston2016approximate}, $\mathcal{H}_2$ is negative semi-definite and therefore preserves the concave curvature structure required for stable maximization of the objective. By removing the positive-semidefinite outer product component $\mathcal{H}_1$, ACGN2 avoids introducing convex curvature bias into the actor update. The Hessian-Vector product of ACGN2 approximation would be given by:
\begin{equation}
\tilde{\mathcal{H}}_{\text{ACGN2}}\cdot\Delta = \mathcal{H}_2\cdot\Delta
\end{equation}
While $\mathcal{H}_2$ is theoretically negative semi-definite under idealized conditions, in practice the stochastic estimation of Hessian-vector products and numerical approximations (e.g., conjugate gradient) may introduce positive eigen-values rendering an indefinite $\mathcal{H}_2$, yielding directions that do not strictly preserve this property. To account for this, we filter the updates based on the observed curvature such that only negative semi-definite $\mathcal{H}_2$ participate in the updates, ensuring stable optimization. Additionally we also add a ridge term $\lambda I$ where $-\infty < \lambda < 0$ to ensure $\mathcal{H}_2 + \lambda I$ is negative definite. Overall, this formulation leverages the intrinsic geometry of the policy while remaining robust to practical estimation noise, leading to stable and reliable updates.

\section{Experimental Setup and Results}

We conduct experiments to evaluate whether second-order methods yield improved convergence relative to  first-order and natural gradient baselines across diverse benchmark environments. Our evaluation spans both discrete control tasks, including CartPole and LunarLander \cite{brockman2016openai}, and continuous control domains from MuJoCo, including Reacher and Humanoid \cite{todorov2012mujoco}. All experiments are conducted over 5 independent random seeds. We evaluate performance in terms of sample efficiency and convergence behavior.

\subsection{Environment with Discrete action-spaces}

We evaluate the second order two-timescale actor-critic methods using environments with discrete action spaces agent selects from a finite set $\mathcal{A}$ of predefined control actions (discrete) rather than continuous-valued forces, in each step of the episode.

First, we perform experiments on the CartPole environment provided by the OpenAI gym \cite{brockman2016openai}, where the observation space is defined as $\mathcal{S} \subset \mathbb{R}^4$, consisting of the state vector $(x, \dot{x}, \omega, \dot{\omega})$. Here, $x$ denotes the cart position, $\dot{x}$ is the cart velocity, $\omega$ the pole angle, and $\dot{\omega}$ denotes the pole angular velocity. These variables are bounded according to the environment specifications. The initial state at the beginning of each episode is sampled independently from the uniform interval $(-0.05, 0.05)$, with randomness controlled via different seeds across runs. The action space is discrete given by $\mathcal{A} = \{0, 1\}$, where the actions correspond to applying a fixed magnitude force to push the cart either left or to the right. 

The reward function assigns a value of $+1$ at every time step until termination. An episode terminates when the pole angle exceeds a threshold i.e $|\omega| > 12^\circ$, or when the cart position exceeds $|x| > 2.4$. Consequently the cumulative reward is directly proportional to the episode length, since a reward of $+1$ is received at each time step until termination. In our setup, we predefined the maximum episode length as 500 steps, with a maximum achievable reward of 500.

\begin{figure}[h]
    \centering
    \includegraphics[width=0.8\linewidth]{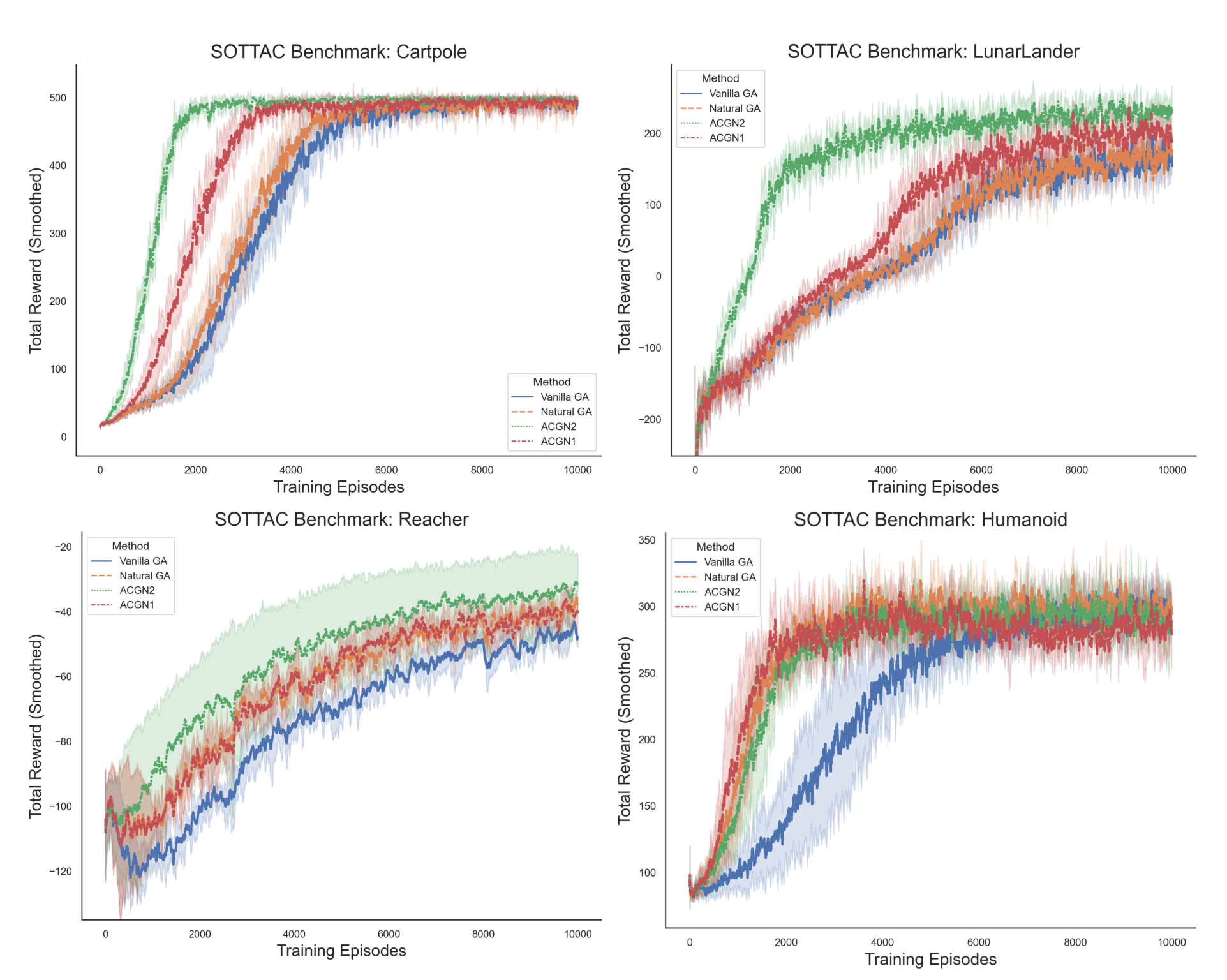}
    \caption{SOTTAC (Second-Order Two-Timescale Actor-Critic) benchmark results across four RL environments: CartPole and LunarLander (discrete action spaces), and Reacher and Humanoid (continuous action spaces). We compare Vanilla Gradient Ascent (REINFORCE, blue), Natural Gradient Ascent (green), and the proposed second-order variants ACGN1 (red) and ACGN2 (orange).}
    \label{fig:sottac_benchmarks}
\end{figure}

We also evaluate our methods on the Lunar-Lander environment from OpenAI gym \cite{brockman2016openai}. Here, the observation space is defined as $\mathcal{S} \subset \mathbb{R}^6 \times \{0,1\}^2$, consisting of the state vector $(x, y, \dot{x}, \dot{y}, \theta, \omega, b_1, b_2)$, where $(b_1, b_2) \in \{0,1\}^2$, where $(x, y)$ denote the lander position, $(\dot{x}, \dot{y})$ the linear velocities, $\theta, \omega$ represent its angle and angular velocity, and $(b_1, b_2)$ are two booleans indicate whether each leg is in contact with the ground or not. The action space is discrete, given by $\mathcal{A} = \{0, 1, 2, 3\}$, where the actions correspond to firing the left orientation engine, main engine, right orientation engines or taking no action.

At each time step, a scalar reward is assigned, and the episode return is the sum of the step-wise rewards. The reward encourages stable landing by favoring proximity, low velocity, and upright orientation, while penalizing fuel usage. A bonus of $+10$ is given per leg contact and a terminal reward of $+100$ or $-100$ is assigned for successful landing or crash, respectively. Unlike CartPole, the LunarLander environment does not admit a fixed maximum reward due to its shaped reward structure. In practice, performance is evaluated based on average return, and the environment is considered solved when the mean reward over 100 consecutive episodes exceeds 200.

We perform simulation experiments by assuming the linear softmax (Gibbs) policy parameterization which is known to be log-concave, where the Hessian of the log probabilities is known to be negative-semidefinite. The latter parameterization provides the distribution to sample the action as follows:
\begin{equation}
    \pi(a \mid s ; \theta) =
    \frac{\exp(\phi(s,a)^\top \theta)}
    {\sum_{b \in \mathcal{A}} \exp(\phi(s,b)^\top \theta)} ,
\end{equation}
where $\phi(s,a)$ is a one-hot encoded feature vector w.r.t the discrete action $a$, with the state vector $s$ embedded in the corresponding action component, and $\theta$ is the trainable parameter vector of dimension $\text{dim}(\mathcal{S}) \times |\mathcal{A}|$.

Across CartPole and Lunar-Lander environments, ACGN2 consistently exhibits faster convergence relative and improved sample efficiency relative to the baseline models. ACGN1 also demonstrates improved performance over first-order and natural gradient approaches, although it converges slightly slower compared to ACGN2. This behavior can be attributed to the quality of curvature approximation. ACGN2 relies on the intrinsic curvature component in \eqref{eq:hessian-decomposition}. which preserves the negative semidefinite induced by log-concave policy. As a result it provides well-aligned second-order updates. In contrast, ACGN1 incorporates additional outer product term as in \eqref{eq:gauss-newton1-approximation} which introduces a convex component into the concave structure for stable maximization, resulting in slightly slower convergence. Natural gradient methods which rely solely on Fisher-information based updates which are more conservative leading to slower learning. Vanilla gradient ascent (REINFORCE) shows the slowest convergence due to the high-variance updates in absence of curvature information. Overall, these results highlight the importance of accurately capturing the intrinsic curvature of the policy landscape for efficient and stable policy optimization in the discrete action space.

\subsection{Environment with Continuous action-spaces}

In contrast to discrete action settings, continuous environments are characterized by action spaces $\mathcal{A} \subset \mathbb{R}^d$, where the agent selects continuous control inputs rather than choosing from a finite action set. We evaluate on Reacher and Humanoid (MuJoCo environments) \cite{todorov2012mujoco}.

The Reacher environment consists of a two-jointed robotic arm operating in a planar 2-D space. The objective is to move the end-effector (tip) of the arm to a target position. The reward is defined as the negative Euclidean distance between the end-effector and the target, along with an additional penalty on the magnitude of control inputs, encouraging precise and energy-efficient movements.

The Humanoid environment models a high-dimensional 3-D bipedal robot with articulated limbs, including a torso, legs, and arms. The goal is to learn stable locomotion by maximizing forward velocity while maintaining balance. The reward function primarily incentivizes forward movement, while incorporating penalties for large control inputs and deviations from an upright posture.

We use a Gaussian policy parametrized by a multi-layer perceptron (MLP). The policy is defined as
\begin{equation}
\pi(a \mid s; \theta) = \mathcal{N}(\mu_\theta(s), \sigma_\theta^2(s)),
\end{equation}
where both the mean $\mu_\theta(s)$ and standard deviation $\sigma_\theta(s)$ are learned functions of the state. Since the policy is not strictly log-concave, the curvature information provides noisy second-order updates. However, $\mu_\theta(s)$ is bounded using $\tanh$ activation and $\sigma_\theta(s)$ is clipped, to maintain numerical stability. As a result, we still observe consistent, albeit moderate improvements over the baselines.

In the Reacher environment, ACGN2 demonstrates relatively faster convergence, but with higher variance due to sensitivity to noisy curvature estimates arising from Hessian-vector computations. This variability also affects ACGN1, leading it to converge comparably to Natural Gradient Ascent. In the Humanoid environment, the higher dimensionality of the continuous action space further amplifies the effect of noisy curvature estimates. In this setting, ACGN1 and Natural Gradient exhibit relatively more stable and faster convergence, while ACGN2 shows comparable performance with slightly higher variability, but converges to a relatively higher final return. Nevertheless, all second-order methods consistently outperform REINFORCE, indicating the benefit of incorporating curvature information even under non-concave policies.

\section{Conclusion}
We showed that second-order approximations for two-timescale actor--critic methods effectively leverage curvature information for policy optimization, improving convergence and sample efficiency over first-order and natural gradient baselines across discrete and continuous control tasks. Future work will focus on scalable approximations and extending these curvature approximations to trust-region and proximal policy optimization methods such as TRPO~\cite{schulman2015trpo} and PPO~\cite{schulman2017ppo}. Such extensions could enable more accurate curvature-aware trust-region updates while preserving the stability guarantees of constrained and clipped policy optimization. Furthermore, they may provide a practical pathway for incorporating second-order information into large-scale reinforcement learning algorithms.

%
%
%
\newpage
\bibliography{refs}
\bibliographystyle{plainnat}

\newpage
\appendix
\input{appendix.tex}

\newpage

\end{document}

%% file: appendix.tex
\section{Derivation of the Policy Hessian Decomposition}

Here, we derive the policy Hessian decomposition used in Eq \eqref{eq:hessian-decomposition}. We begin with the standard policy gradient expression for the expected return:
\begin{equation}
\nabla J(\theta) = \mathbb{E}_{(s,a) \sim \rho_\gamma} 
\left[ Q^\pi(s,a) \nabla_\theta \log \pi_\theta(a|s) \right].
\end{equation}
Taking the second derivative with respect to $\theta$, we obtain:
\begin{equation}
\nabla^2 J(\theta) = \nabla_\theta \mathbb{E}_{\rho_\gamma} 
\left[ Q^\pi(s,a) \nabla_\theta \log \pi_\theta(a|s) \right].
\end{equation}
Using the product rule, we expand:
\begin{equation}
\begin{aligned}
\nabla^2 J(\theta) =
\mathbb{E}_{\rho_\gamma}
\left[
Q^\pi(s,a) \nabla^2_\theta \log \pi_\theta(a|s)
\right]
+
\mathbb{E}_{\rho_\gamma}
\left[
\nabla_\theta Q^\pi(s,a) \nabla_\theta \log \pi_\theta(a|s)^\top
\right]
\\
+
\mathbb{E}_{\rho_\gamma}
\left[
Q^\pi(s,a)
\nabla_\theta \log \pi_\theta(a|s)
\nabla_\theta \log \pi_\theta(a|s)^\top
\right].
\end{aligned}
\end{equation}
Additionally, differentiating the log-policy gradient term yields an asymmetric hessian due to the interaction-term $\mathcal{H}_{12}$. The interaction term contributes both
\[
E[\nabla_\theta Q^\pi (\nabla_\theta \log\pi)^T]
\]
and its transpose
\[
E[\nabla_\theta \log\pi (\nabla_\theta Q^\pi)^T],
\]
yielding the symmetric decomposition
\begin{equation}
\nabla^2 J(\theta) =
\mathcal{H}_1(\theta) + \mathcal{H}_2(\theta) + \mathcal{H}_{12}(\theta) + \mathcal{H}_{12}(\theta)^\top,
\end{equation}
where the components are defined as:
\begin{align}
H_1(\theta) &= 
\mathbb{E}_{\rho_\gamma}
\left[
Q^\pi(s,a)
\nabla_\theta \log \pi_\theta(a|s)
\nabla_\theta \log \pi_\theta(a|s)^\top
\right], \\
H_2(\theta) &=
\mathbb{E}_{\rho_\gamma}
\left[
Q^\pi(s,a)
\nabla^2_\theta \log \pi_\theta(a|s)
\right], \\
H_{12}(\theta) &=
\mathbb{E}_{\rho_\gamma}
\left[
\nabla_\theta \log \pi_\theta(a|s)
\nabla_\theta Q^\pi(s,a)^\top
\right].
\end{align}
Therefore, given above is the decomposition of the policy Hessian into (i) an outer-product term $H_1$, (ii) an intrinsic curvature term $H_2$, and (iii) interaction terms $H_{12}$ and $H_{12}^\top$.

\section{The Two-timescale Insight}
\label{app:two-timescale-insight}
Under the two-timescale actor-critic framework \cite{konda2000actorcritic, borkar1997stochastic}, $\mathcal{Q}$ is estimated using a separate function network known as the critic. The critic value estimate ($Q^{c}$) evolves on a faster timescale than the actor. As a result, during policy (actor) updates the critic can be treated as approximately stationary with respect to the actor parameters, i.e., almost converged (quasi-stationary) \cite{konda2000actorcritic}. This motivates the approximation $\nabla Q^{c} \approx 0$, yielding $H^{AC}_{12} \approx 0$ and simplifying the second term in \eqref{eq:hessian-combined-decomposition} (in HVP form \eqref{eq:H1H2-HVPform}) to just $\mathcal{H}_2\cdot\Delta$. In order to prove the above insight, we make assumptions regarding the smoothness of the parameterized policy $\pi_{\theta}$ and the parameterized critic $Q^{c}$. Let $\omega$ be the parameters of the critic network and let $\|\cdot\|$ denote the $l_{2}-$norm for vectors and operator norm for the matrices. 
\\ \\
\noindent \textbf{(A1) Smooth Objective.}
The objective function \(J:\mathbb{R}^{d}\rightarrow\mathbb{R}\) is twice continuously differentiable with \(L\)-Lipschitz continuous gradients and \(\rho\)-Lipschitz continuous Hessian. That is, there exist constants \(L,\rho>0\) such that, for all \(\theta,\theta'\in\mathbb{R}^{d}\),
\begin{equation}
\|\nabla J(\theta)-\nabla J(\theta')\|
\le
L\|\theta-\theta'\|,
\end{equation}
and
\begin{equation}
\|\nabla^2J(\theta)-\nabla^2J(\theta')\|
\le
\rho\|\theta-\theta'\|.
\end{equation}

\noindent \textbf{(A2) (Parameterization regularity):} The policy $\pi_{\theta}(a \mid s)$ is twice continuously differentiable w.r.t $\theta$ which implies for any choice of the parameter $\theta$, any state-action pair (s,a), there exist a constraint $0 < C_g, C_H < \infty$ such that
\begin{equation}
    \|\nabla \text{log} \pi_{\theta}(a \mid s)\| \leq C_g, \, \text{and} \ \, \|\nabla^{2} \text{log} \pi_{\theta}(a \mid s)\| \leq C_H \quad \forall \quad (s,a) \in |\mathcal{S}| \times |\mathcal{A}|
\end{equation}
\noindent \textbf{(A3) (Smooth critic):} The critic $Q(s,a;\omega)$ is differentiable in $\omega$ and its parameter gradient is uniformly bounded as
\begin{equation}
    \|\nabla_{\omega} Q^{c}(s,a;\omega)\| \leq C_c \quad \forall \quad (s,a) \in |\mathcal{S}| \times |\mathcal{A}|
\end{equation}
\noindent \textbf{(A4) (Critic tracking):} For each actor parameter $\theta$, let $\omega^\star(\theta)$ denote the critic parameter that is optimal for the current policy $\pi_\theta$. Under two-timescale actor--critic updates, the almost converged quasi-stationary critic parameter tracks this optimum with error
\begin{equation}
\| \omega - \omega^\star(\theta) \| \leq \epsilon_c,
\end{equation}
where $\epsilon_c$ is small when the critic timescale is sufficiently faster than the actor timescale.
\\ \\
\noindent \textbf{(A5) (Slow actor variation):} During actor update, the critic parameter $\omega$ is considered fixed with respect to the policy (actor) parameter $\theta$:
\begin{equation}
\frac{\partial \omega}{\partial \theta} \approx 0.
\end{equation}

Based on the above assumptions, we now show that the actor--critic interaction term becomes small under the above assumptions.

\textbf{Proposition 1.} \textit{Under Assumptions A1--A5, the actor--critic interaction term}
\begin{equation}
\mathcal{H}_{12}^{AC}(\theta)
=
\mathbb{E}_{\rho_\gamma}
\left[
\nabla_\theta \log \pi_\theta(a|s)
\nabla_\theta Q_c(s,a;\omega)^\top
\right]
\label{eq:actor-critic-interaction-term}
\end{equation}
\textit{satisfies}
\begin{equation}
\|H_{12}^{AC}(\theta)\| \leq C_g C_c \left\|\frac{\partial \omega}{\partial \theta}\right\|.
\end{equation}
\textit{In particular, under the two-timescale quasi-stationary approximation
$\partial \omega / \partial \theta \approx 0$, we obtain $
H_{12}^{AC}(\theta) \approx 0$}

\begin{proof}
Using the chain rule, the dependence of the critic on the actor parameters is given by
\begin{equation}
\nabla_\theta Q^{c}(s,a;\omega)
=
\frac{\partial}{\partial \theta}Q^{c}(s,a;\omega)
=
\left(\frac{\partial \omega}{\partial \theta}\right)^\top
\frac{\partial}{\partial \omega}Q^{c}(s,a;\omega)
=
\left(\frac{\partial \omega}{\partial \theta}\right)^\top
\nabla_\omega Q^{c}(s,a;\omega).
\end{equation}
Therefore Eq \eqref{eq:actor-critic-interaction-term} can be written as 
\begin{align}
\mathcal{H}_{12}^{AC}(\theta)
=
\mathbb{E}_{\rho_\gamma}
\left[
\nabla_\theta \log \pi_\theta(a|s)
\left(\left(\frac{\partial \omega}{\partial \theta}\right)^\top
\nabla_\omega Q^{c}(s,a;\omega)\right)^\top
\right]
\\
=
\sum_{(s,a)}{\rho_{\gamma}(s,a;\theta)\nabla_\theta \log \pi_\theta(a|s)
\nabla_\omega Q^{c}(s,a;\omega)^\top\frac{\partial \omega}{\partial \theta}}
\end{align}
Taking the operator-norm and applying Cauchy-Swartz inequality \cite{rudin1976principles}:
\begin{align}
\|\mathcal{H}_{12}^{AC}(\theta)\|
&\leq
\sum_{(s,a)}\rho_{\gamma}(s,a;\theta)
\left\|
\nabla_\theta \log \pi_\theta(a|s)
\nabla_\omega Q^{c}(s,a;\omega)^\top\frac{\partial \omega}{\partial \theta}
\right\|
\\
&=
\sum_{(s,a)}\rho_{\gamma}(s,a;\theta)
\left\|
\nabla_\theta \log \pi_\theta(a|s)\right\|
\left\|
\nabla_\omega Q^{c}(s,a;\omega)^\top
\right\|
\left\|
\frac{\partial \omega}{\partial \theta}
\right\|
\\
&\leq
\mathbb{E}_{\rho_\gamma}
\left[
C_g
C_c
\left\|
\frac{\partial \omega}{\partial \theta}
\right\|
\right]
=
C_g C_c
\left\|
\frac{\partial \omega}{\partial \theta}
\right\| (\text{After normalization w.r.t} \sum_{(s,a)}\rho_{\gamma}(s,a;\theta)).
\end{align}
Under a two-timescale actor--critic update, the critic evolves on a faster timescale and is treated as quasi-stationary during actor updates. So, based on (A4)-(A5) $
\frac{\partial \omega}{\partial \theta} \approx 0$,
which implies $\mathcal{H}_{12}^{AC}(\theta) \approx 0$.
\end{proof}

\noindent Using Proposition~\ref{prop:h12_bound}, the policy Hessian in the actor--critic setting can be approximated as
\begin{align}
\nabla^2 J(\theta)
&=
\mathcal{H}_1(\theta) + \mathcal{H}_2(\theta) + \mathcal{H}_{12}(\theta) + \mathcal{H}_{12}(\theta)^\top \\
&\approx
\mathcal{H}_1(\theta) + \mathcal{H}_2(\theta).
\end{align}
Thus, under the two-timescale approximation, the interaction-free curvature estimate is
\begin{equation}
\widetilde{\mathcal{H}}^{AC}(\theta) = H_1(\theta) + H_2(\theta).
\end{equation}
This justifies Hessian approximation on the second-order actor--critic update that we have proposed in this paper. The result above does not claim that the true policy-dependent value function $Q^\pi(s,a;\theta)$ is independent of $\theta$. Rather, we state that in the two-timescale actor-critic implementation, when the critic is represented by a separate parameter vector $\omega$ and is updated on a faster timescale, the critic estimate $Q_c(s,a;\omega)$ can be treated as fixed w.r.t $\theta$. Based on Proposition \ref{prop:h12_bound}, we prove that in the two-timescale actor-critic method, the interaction term involving $\nabla_\theta Q^c$ is suppressed enabling a stable HVP computation of the outer product term $\mathcal{H}_1$ and intrinsic curvature term $\mathcal{H}_2$ without incurring the high-variance interaction term $H_{12}$.

\section{Curvature Properties of the Policy Hessian}
\label{app:curvature}

Once we remove the policy Hessian using the two-timescale insight, we approximate the policy Hessian as ACGN1 and ACGN2 as mentioned in Section \ref{subsec:acgn1} and \ref{subsec:acgn2}. In this section, we analyze the curvature structure of the policy Hessians and after approximations used in ACGN1 and ACGN2. Specifically, we characterize when the outer product $\mathcal{H}_{1}$ is positive semidefinite and when the intrinsic curvature term $\mathcal{H}_{2}$ is negative semidefinite.

\paragraph{Hessian decomposition:}
The policy gradient can be written using the likelihood-ratio identity as mentioned in \eqref{eq:policygradient}:
\begin{equation}
\nabla_\theta J(\theta)
=
\mathbb{E}_{s,a \sim \pi_\theta}
\left[
Q^{\pi_\theta}(s,a)
\nabla_\theta \log \pi_\theta(a|s)
\right].
\end{equation}
Differentiating again (ignoring actor--critic interaction terms based on the two-timescale insight) yields the dominant policy curvature components:
\begin{align}
\mathcal{H}_1
&=
\mathbb{E}_{s,a}
\left[
Q^{c}(s,a)
\nabla_\theta \log \pi_\theta(a|s)
\nabla_\theta \log \pi_\theta(a|s)^\top
\right],
\\
\mathcal{H}_2
&=
\mathbb{E}_{s,a}
\left[
w(s,a)
\nabla_\theta^2 \log \pi_\theta(a|s)
\right],
\end{align}

\paragraph{Positive semidefiniteness of $\mathcal{H}_1$:}
For any vector $v \in \mathbb{R}^d$,
\begin{equation}
v^\top \mathcal{H}_1 v
=
\mathbb{E}_{s,a}
\left[
w(s,a)
\left(
v^\top \nabla_\theta \log \pi_\theta(a|s)
\right)^2
\right].
\end{equation}
where $w(s,a)$ can either be $Q^{c}(s,a)$ or $A(s,a)$.
Since $(v^\top \nabla_\theta \log \pi_\theta)^2 \geq 0$, it follows that
\begin{equation}
\mathcal{H}_1 \succeq 0
\quad \text{if } w(s,a) \geq 0.
\end{equation}
Thus, $\mathcal{H}_1$ is a weighted outer-product (Fisher-like) term and is
positive semidefinite under nonnegative weighting. To ensure non-negative weighting we first use a preprocessing step during optimization in environment with non negative rewards, where we subtract with the minimum possible reward to make the reward positive before optimization. Secondly, for the ACGN1 where we use advantage estimates, we use a small learning rate for convergence as mentioned in \cite{furston2016approximate}. 
Thirdly, we add a ridge term or a damping term $\lambda I$, where $-\infty < \lambda < 0$, to ensure that the outer product term $\mathcal{H}_{\text{ACGN1}} + \lambda I$ is negative-definite. Additionally, we set a screening step, where only the negative-definite Hessian estimates are allowed to update in the Newton step. The first 3 steps minimize the probability of the outer product term becoming indefinite, and the last step is to ensure high variance updates caused due to indefinite hessians as the curvature information will be extremely noisy and misleading in indefinite Hessians.

\paragraph{Log-concavity and curvature of $\mathcal{H}_2$.}
A policy is log-concave in $\theta$ if $\log \pi_\theta(a|s)$ is concave, which implies
\begin{equation}
\nabla_\theta^2 \log \pi_\theta(a|s) \preceq 0.
\end{equation}
Under this condition,
\begin{equation}
v^\top \mathcal{H}_2 v
=
\mathbb{E}_{s,a}
\left[
w(s,a)\,
v^\top \nabla_\theta^2 \log \pi_\theta(a|s) v
\right]
\leq 0
\end{equation}
whenever $w(s,a) \geq 0$. Therefore, $
\mathcal{H}_2 \preceq 0$ for log-concave policies with nonnegative weights. 

\paragraph{Examples of log-concave policies.}
\begin{itemize}
    \item \textbf{Softmax policy:}
    For linear logits $z_\theta(s,a) = \theta_a^\top \phi(s)$,
    \begin{equation}
    \log \pi_\theta(a|s)
    =
    \theta_a^\top \phi(s)
    -
    \log \sum_{a'} \exp(\theta_{a'}^\top \phi(s)),
    \end{equation}
    where, the log-sum-exp term is convex, and its negative is concave. $\implies \nabla^2 \log \pi_\theta \preceq 0$. Since we take log before estimating curvature to be concave or convex, we call it log-concave policy.

    \item \textbf{Gaussian policy with MLP mean and learned variance:} For continuous-control environments, we use a Gaussian policy parameterized by an MLP:
\[
\pi_\theta(a\mid s)
=
\mathcal{N}\!\left(a;\mu_\theta(s),\sigma_\theta^2(s)\right),
\]
where both the mean \(\mu_\theta(s)\) and the standard deviation
\(\sigma_\theta(s)\) are learned functions of the state. In contrast to the
fixed-covariance linear Gaussian case, this policy is not guaranteed to be
globally log-concave in \(\theta\), since both the nonlinear mean network and the
learned variance contribute to the curvature of the log-density. Nevertheless, the log-density retains a locally quadratic structure in the
action residual \(a-\mu_\theta(s)\), and the intrinsic curvature term remains
informative for local policy improvement. To improve numerical stability, we
bound the mean output using a \(\tanh\) activation and clip the standard
deviation to avoid degenerate variances. These stabilizations reduce extreme
curvature values and make the damped Hessian/Gauss--Newton estimates more
reliable in practice, yielding consistent but moderate gains in continuous
control environments.
\end{itemize}

\section{Convergence Analysis}
\label{subsec:convergence}

We now analyze the convergence behavior of the proposed damped second-order
actor update. Assumptions (A1)--(A5) ensure smoothness of the objective, bounded
policy and critic derivatives, and the validity of the two-timescale
quasi-stationary critic approximation. In addition, for the convergence
argument, we require the damped curvature estimate used in the actor update to
be uniformly negative definite. Therefore, to establish a stable second-order ascent direction and derive
an admissible learning-rate range, we assume that the damped curvature
estimate is uniformly negative definite and that the policy gradient is
uniformly bounded.

\noindent\textbf{(A6) Bounded gradients.}
There exists \(G>0\) such that
\[
\|g_t\|\le G
\qquad
\text{for all }t .
\]

\noindent \textbf{(A7) Negative-definite damped curvature.}
Let the damped estimated Hessian $\widehat{H}_t \in \mathbb{R}^{d \times d}$ used in the actor update be calculated from $\widetilde{H}_t  \in \mathbb{R}^{d \times d}$ using a ridge-term as follows:
\[
\widehat{H}_t = \widetilde{H}_t + \lambda I .
\]
We assume that there exist constants $0<m<M<\infty$ such that
\[
-M||x||^2 \leq x^\top\widehat{H}_tx \leq -m||x||^2 .
\]
where $x \in \mathbb{R}^{d} $ be any arbitrary vector.
Equivalently,
\[
-\frac{||x||^2}{m} \leq x^\top\widehat{H}_t^{-1}x \leq -\frac{||x||^2}{M}.
\]
This assumption states that, after damping, only well-conditioned negative
definite curvature estimates are used for the actor update. But damping alone does not guarantee
negative-definiteness. So we filter out indefinite updates that persist after damping.

\paragraph{Hessian approximation error.}

Let $H_t$ denote the true curvature term and $\widehat{H}_t$ denote the practical
damped Gauss--Newton curvature estimate. The error in the curvature estimate is
bounded as
\[
\|\widehat{H}_t - H_t\|
\leq
\epsilon_H + \epsilon_{GN} + \epsilon_{\mathrm{damp}},
\]
where $\epsilon_H$ denotes the critic-induced Hessian error, $\epsilon_{GN}$
denotes the Gauss--Newton approximation error, and
$\epsilon_{\mathrm{damp}}$ denotes the bias introduced by damping. From
(A2)--(A5), the critic-tracking and two-timescale approximation imply that the
actor--critic interaction term satisfies
\[
\|\mathcal{H}_{12}^{AC}(\theta)\|
\leq
C_gC_c
\left\|
\frac{\partial \omega}{\partial \theta}
\right\|
\approx 0.
\]
Thus, the dominant practical curvature approximation error is controlled by
critic tracking, Gauss--Newton approximation, and damping.

\begin{proposition}[Exact learning-rate bound for damped second-order ascent]
\label{prop:damped_second_order_convergence}
Suppose assumptions (A1)--(A7) hold. Consider the damped second-order actor
update
\[
\theta_{t+1}
=
\theta_t
-
\alpha_t\widehat H_t^{-1}g_t .
\]
If the learning rate satisfies
\[
0<\alpha_t<\alpha_{\max},
\]
where
\[
\alpha_{\max}
=
\frac{3m^2}{\rho G}
\left[
-\frac12
+
\sqrt{
\frac14+\frac{2\rho G}{3Mm}
}
\right],
\]
then the update yields monotonic improvement up to critic, curvature,
Gauss--Newton, and damping approximation errors. Moreover, for constant
\(\alpha_t=\alpha<\alpha_{\max}\),
\[
\min_{0\le t<T}\|\nabla J(\theta_t)\|^2
=
O\left(\frac1T\right)
+
O(\epsilon_C+\epsilon_H+\epsilon_{GN}+\epsilon_{\mathrm{damp}}),
\]
and therefore
\[
\min_{0\le t<T}\|\nabla J(\theta_t)\|
=
O\left(\frac1{\sqrt T}\right).
\]
\end{proposition}

\begin{proof}
Using the third-order Taylor lower bound from Assumption (A1), we have
\[
J(\theta_{t+1})
\ge
J(\theta_t)
+
g_t^\top \Delta_t
+
\frac12 \Delta_t^\top H_t\Delta_t
-
\frac{\rho}{6}\|\Delta_t\|^3
-
O(\epsilon_C+\epsilon_H+\epsilon_{GN}+\epsilon_{\mathrm{damp}}).
\]
Replacing \(H_t\) by the damped curvature estimate \(\widehat H_t\), with the
resulting approximation error absorbed into
\(O(\epsilon_C+\epsilon_H+\epsilon_{GN}+\epsilon_{\mathrm{damp}})\), gives
\[
J(\theta_{t+1})
\ge
J(\theta_t)
+
g_t^\top \Delta_t
+
\frac12 \Delta_t^\top \widehat H_t\Delta_t
-
\frac{\rho}{6}\|\Delta_t\|^3
-
O(\epsilon_C+\epsilon_H+\epsilon_{GN}+\epsilon_{\mathrm{damp}}).
\]
Substituting
\[
\Delta_t=-\alpha_t\widehat H_t^{-1}g_t
\]
yields
\begin{align}
J(\theta_{t+1})
&\ge
J(\theta_t)
-
\alpha_t g_t^\top \widehat H_t^{-1}g_t
+
\frac{\alpha_t^2}{2}
g_t^\top \widehat H_t^{-1}g_t
-
\frac{\rho\alpha_t^3}{6}
\|\widehat H_t^{-1}g_t\|^3
-
O(\epsilon_C+\epsilon_H+\epsilon_{GN}+\epsilon_{\mathrm{damp}}).
\end{align}

By Assumption (A7),
\[
g_t^\top \widehat H_t^{-1}g_t
\le
-\frac1M\|g_t\|^2,
\]
and
\[
g_t^\top \widehat H_t^{-1}g_t
\ge
-\frac1m\|g_t\|^2.
\]
Moreover,
\[
\|\widehat H_t^{-1}g_t\|
\le
\frac1m\|g_t\|.
\]
Therefore,
\begin{align}
J(\theta_{t+1})-J(\theta_t)
&\ge
\frac{\alpha_t}{M}\|g_t\|^2
-
\frac{\alpha_t^2}{2m}\|g_t\|^2
-
\frac{\rho\alpha_t^3}{6m^3}\|g_t\|^3
-
O(\epsilon_C+\epsilon_H+\epsilon_{GN}+\epsilon_{\mathrm{damp}}).
\end{align}
Using \(\|g_t\|\le G\) from Assumption (A6), we obtain
\[
J(\theta_{t+1})-J(\theta_t)
\ge
\alpha_t\|g_t\|^2
\left(
\frac1M
-
\frac{\alpha_t}{2m}
-
\frac{\rho\alpha_t^2G}{6m^3}
\right)
-
O(\epsilon_C+\epsilon_H+\epsilon_{GN}+\epsilon_{\mathrm{damp}}).
\]

Hence, monotonic improvement up to approximation error is guaranteed whenever
\[
\frac1M
-
\frac{\alpha_t}{2m}
-
\frac{\rho\alpha_t^2G}{6m^3}
>0.
\]
Equivalently,
\[
\frac{\rho G}{6m^3}\alpha_t^2
+
\frac{1}{2m}\alpha_t
-
\frac1M
<0.
\]
Solving this quadratic inequality gives
\[
0<\alpha_t<\alpha_{\max},
\]
where
\[
\alpha_{\max}
=
\frac{
-\frac{1}{2m}
+
\sqrt{
\frac{1}{4m^2}
+
\frac{2\rho G}{3m^3M}
}
}
{\frac{\rho G}{3m^3}}.
\]
Simplifying,
\[
\alpha_{\max}
=
\frac{3m^3}{\rho G}
\left[
-\frac{1}{2m}
+
\sqrt{
\frac{1}{4m^2}
+
\frac{2\rho G}{3m^3M}
}
\right],
\]
and therefore
\[
\boxed{
\alpha_{\max}
=
\frac{3m^2}{\rho G}
\left[
-\frac12
+
\sqrt{
\frac14+\frac{2\rho G}{3Mm}
}
\right].
}
\]

Thus, for every \(0<\alpha_t<\alpha_{\max}\), there exists a constant
\(c_\alpha>0\) such that
\[
\frac1M
-
\frac{\alpha_t}{2m}
-
\frac{\rho\alpha_t^2G}{6m^3}
\ge
c_\alpha .
\]
Consequently,
\[
J(\theta_{t+1})-J(\theta_t)
\ge
\alpha_t c_\alpha \|g_t\|^2
-
O(\epsilon_C+\epsilon_H+\epsilon_{GN}+\epsilon_{\mathrm{damp}}).
\]

Summing from \(t=0\) to \(T-1\), we get
\[
\sum_{t=0}^{T-1}
\alpha_t c_\alpha \|g_t\|^2
\le
J(\theta^\star)-J(\theta_0)
+
O\left(
\sum_{t=0}^{T-1}
(\epsilon_C+\epsilon_H+\epsilon_{GN}+\epsilon_{\mathrm{damp}})
\right),
\]
where \(J(\theta^\star)\) is an upper bound on the objective.

For constant \(\alpha_t=\alpha<\alpha_{\max}\), this gives
\[
\alpha c_\alpha
\sum_{t=0}^{T-1}
\|g_t\|^2
\le
J(\theta^\star)-J(\theta_0)
+
O\left(
T(\epsilon_C+\epsilon_H+\epsilon_{GN}+\epsilon_{\mathrm{damp}})
\right).
\]
Therefore,
\[
\min_{0\le t<T}\|g_t\|^2
\le
\frac{J(\theta^\star)-J(\theta_0)}
{\alpha c_\alpha T}
+
O(\epsilon_C+\epsilon_H+\epsilon_{GN}+\epsilon_{\mathrm{damp}}).
\]
Since \(g_t=\nabla J(\theta_t)\), we obtain
\[
\min_{0\le t<T}\|\nabla J(\theta_t)\|^2
=
O\left(\frac1T\right)
+
O(\epsilon_C+\epsilon_H+\epsilon_{GN}+\epsilon_{\mathrm{damp}}).
\]
Equivalently,
\[
\min_{0\le t<T}\|\nabla J(\theta_t)\|
=
O\left(\frac1{\sqrt T}\right).
\]
If the approximation errors vanish asymptotically, then
\[
\liminf_{t\to\infty}
\|\nabla J(\theta_t)\|^2
=
0.
\]
This proves convergence to a first-order stationary point in the ideal setting,
and convergence to a neighborhood of stationarity under persistent approximation
error.
\end{proof}

\paragraph{Conservative sufficient learning-rate bound.}
Although Proposition~\ref{prop:damped_second_order_convergence} gives the exact
admissible learning-rate range, a simpler sufficient condition can be obtained
by separately controlling the two negative terms in the bracket. Specifically,
if
\[
0<\alpha_t
\le
\min\left\{
\frac{m}{2M},
\sqrt{\frac{3m^3}{2\rho GM}}
\right\},
\]
then
\[
\frac{\alpha_t}{2m}
\le
\frac{1}{4M},
\qquad
\frac{\rho\alpha_t^2G}{6m^3}
\le
\frac{1}{4M}.
\]
Hence,
\[
\frac1M
-
\frac{\alpha_t}{2m}
-
\frac{\rho\alpha_t^2G}{6m^3}
\ge
\frac{1}{2M}>0.
\]
This conservative condition is less tight than the exact discriminant bound but
is often easier to state and tune in practice.

\paragraph{Interpretation.}
The exact bound shows that the admissible learning rate is controlled by the
conditioning of the damped curvature matrix through \(m\) and \(M\), the Hessian
Lipschitz constant \(\rho\), and the gradient bound \(G\). Better conditioned
curvature estimates allow larger stable second-order steps, while large
third-order variation or large gradients require smaller steps. Thus, damping
serves two roles: it stabilizes the inverse curvature direction and directly
controls the allowable step size through the constants \(m\) and \(M\). In the ideal setting where the critic tracks the optimal value function exactly
(\(\epsilon_C \to 0\)) and the curvature estimate is accurate
(\(\epsilon_H, \epsilon_{GN}, \epsilon_{\mathrm{damp}} \to 0\)), the method
recovers convergence to a true stationary point:
\[
\liminf_{t\to\infty} \|\nabla J(\theta_t)\|^2 = 0.
\]

For \textbf{ACGN2}, under log-concave policy parameterizations (e.g., softmax or
Gaussian policies), the intrinsic curvature satisfies
\[
\widetilde{H}_t \preceq 0,
\]
which implies that the negative damped matrix usually used in stochastic optimization (loss minimization) updates:
\[
-\widetilde{H}_t + \lambda I
\]
is naturally positive definite. Consequently, ACGN2 preserves the concave
curvature structure required for stable maximization and admits a well-conditioned
second-order ascent direction. In contrast, for \textbf{ACGN1}, the inclusion of the outer-product term
introduces a positive semi-definite component that can conflict with the
intrinsic curvature, resulting in an indefinite curvature estimate. Therefore,
damping plays a critical role in ensuring that the effective update matrix
remains negative-definite, thereby guaranteeing a stable ascent direction.

Overall, the analysis highlights that while both methods benefit from curvature
information, ACGN2 provides a structurally more stable approximation, whereas
ACGN1 relies more heavily on damping to control curvature-induced instability.

\section{Implementation Details and Runtime Analysis}
\label{sec:appendix_implementation}

\subsection{Hardware Setup and Experimental Protocol}
All experiments were conducted on a CPU-optimized workstation equipped with a multi-core processor and sufficient system memory to support parallel execution. To ensure statistical reliability, each method was evaluated across five independent random seeds: \{42, 100, 2026, 777, 1234\}

We adopt a fully CPU-parallelized evaluation pipeline using Python's \texttt{multiprocessing} and \\
\texttt{concurrent.futures} libraries. Each seed is executed as an independent worker process, enabling near-linear scaling in wall-clock efficiency and eliminating GPU transfer overheads, which non-trivial for low-dimensional control tasks. All experiments enforce deterministic initialization by fixing seeds for PyTorch, NumPy, and Gymnasium environments.

\begin{figure}[h]
    \centering
    \includegraphics[width=0.45\textwidth]{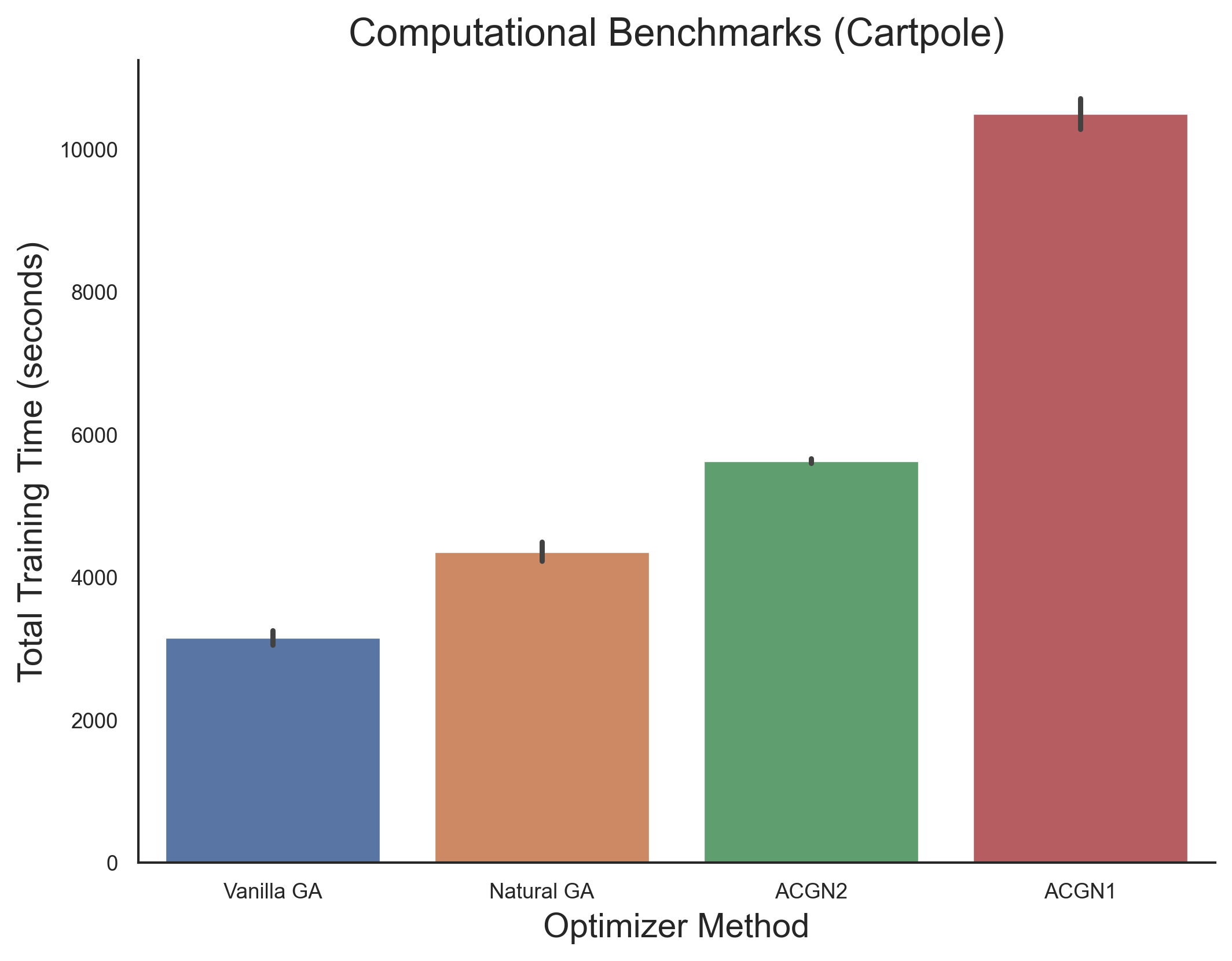}
    \includegraphics[width=0.45\textwidth]{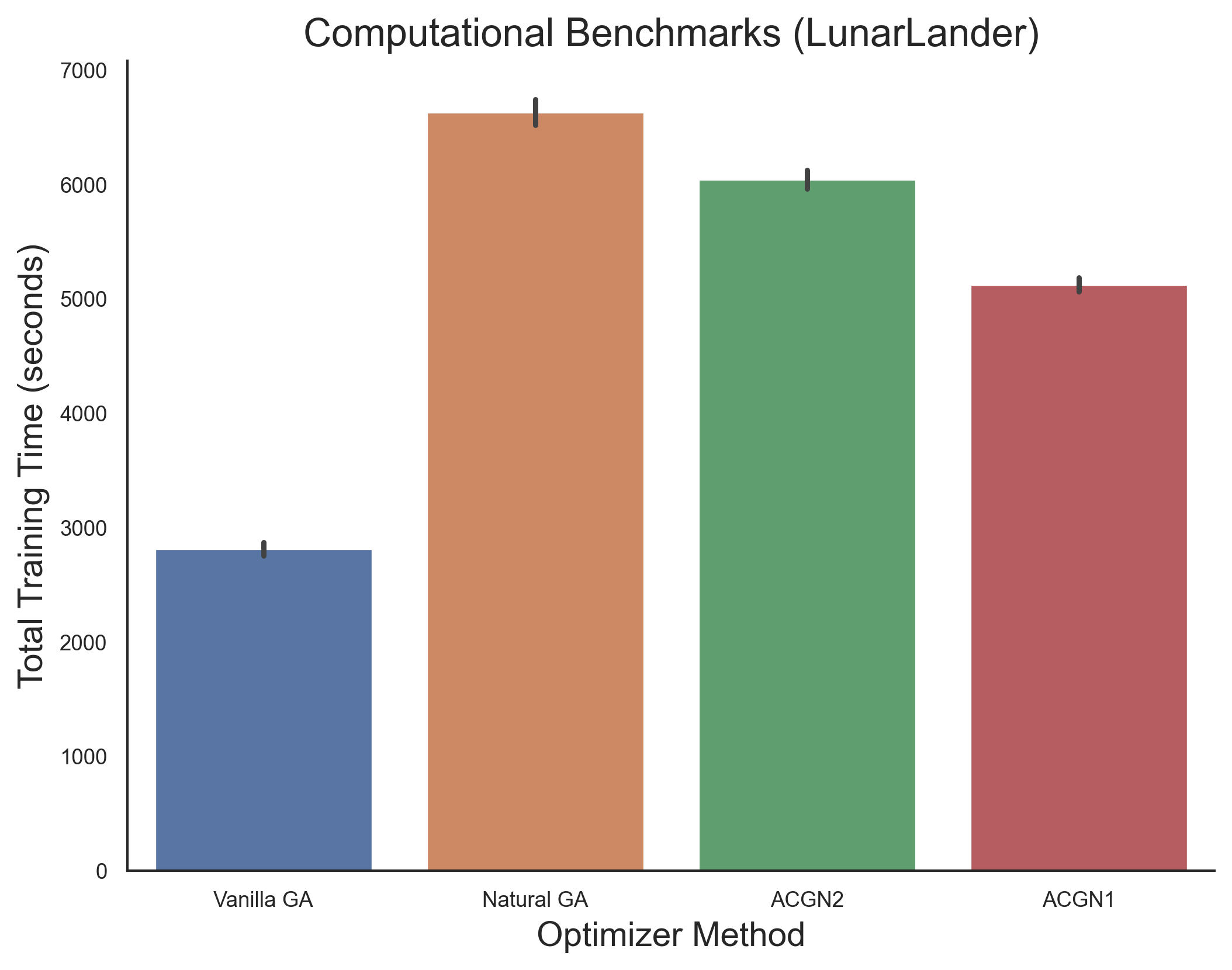}
    
    \vspace{0.4cm}
    
    \includegraphics[width=0.45\textwidth]{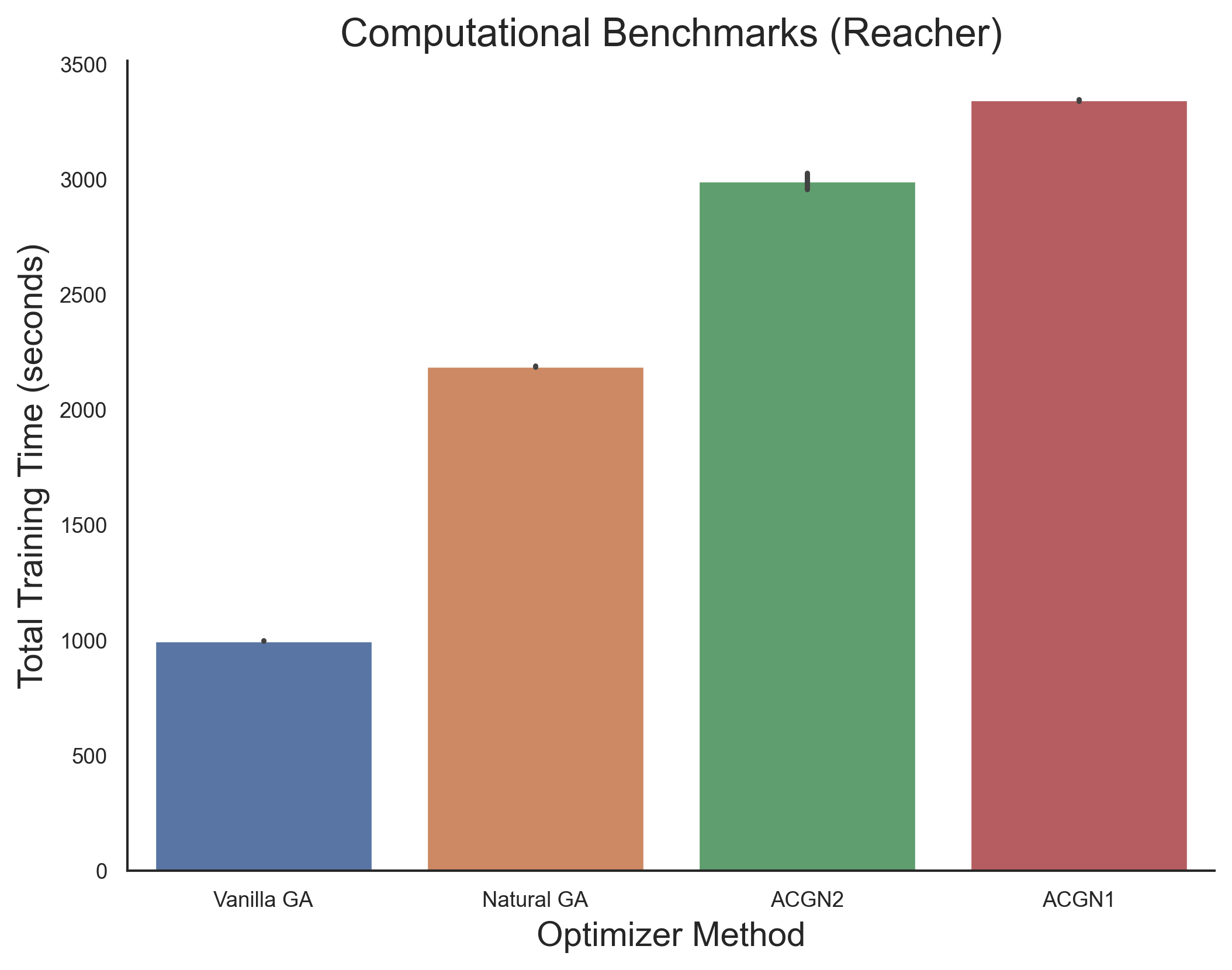}
    \includegraphics[width=0.45\textwidth]{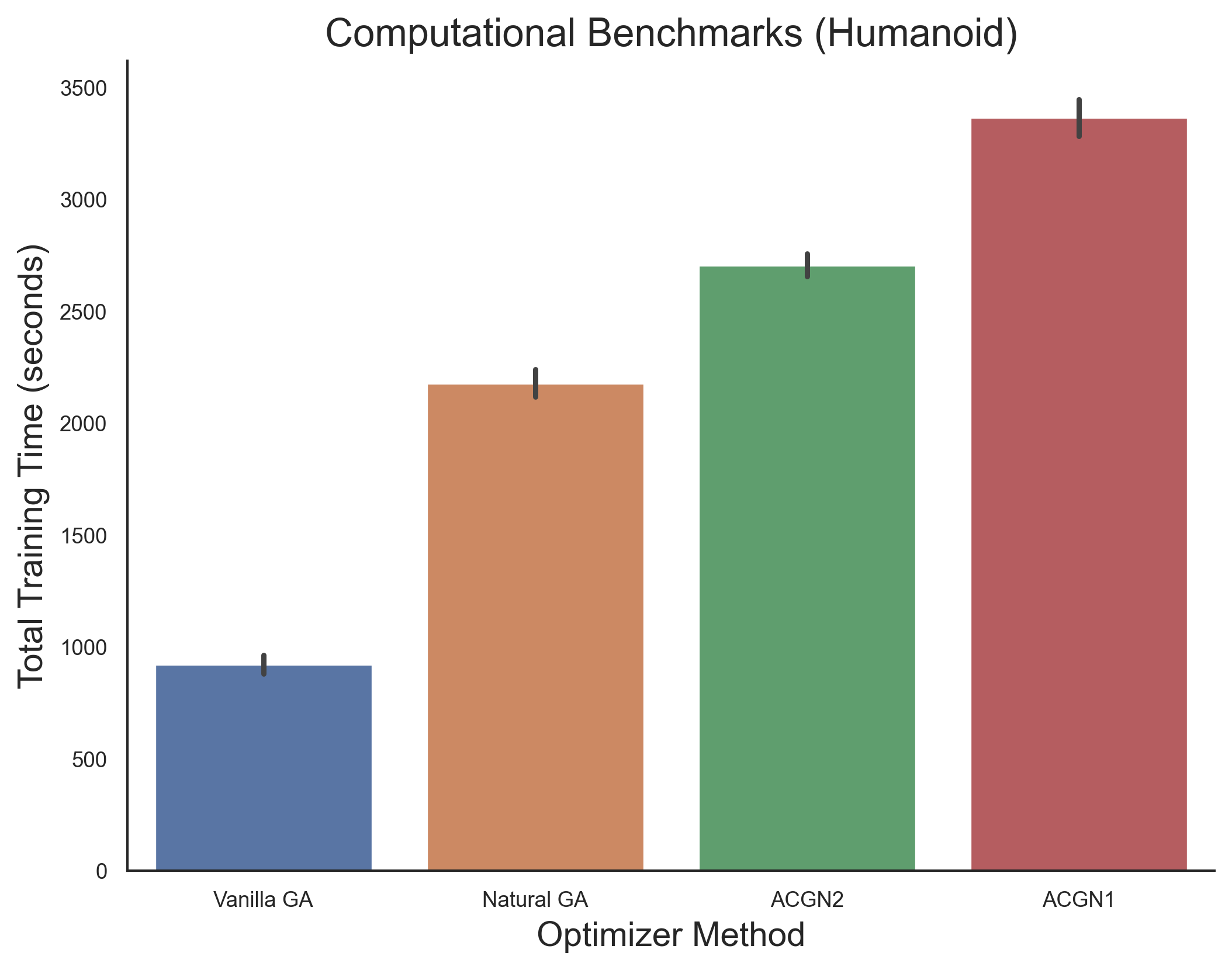}
    
    \caption{
    Wall-clock training time (in seconds) for different optimization methods across four environments: CartPole, LunarLander, Reacher, and Humanoid. Error bars denote standard deviation across five random seeds. 
    Vanilla Gradient Ascent (blue) consistently exhibits the lowest computational cost due to its first-order nature. Natural Gradient (orange) introduces moderate overhead from Fisher matrix computations. The proposed second-order methods ACGN2 (green) and ACGN1 (red) incur higher costs due to curvature estimation and CG-based updates, with ACGN1 being the most expensive owing to its additional outer-product structure.
    }
    \label{fig:runtime_benchmarks}
\end{figure}

\subsection{Runtime Performance and Computational Trade-offs}

Second-order methods introduce additional per-update computational cost due to Fisher/HVP computation, and curvature estimation with damping. While these methods improve sample efficiency, they require higher wall-clock time per iteration compared to first-order approaches such as Vanilla Gradient Ascent (REINFORCE).

The results reveal a clear computational hierarchy over the methods. Vanilla Gradient Ascent remains the best in the computational hierarchy across all environments, reflecting its lightweight first-order updates. Natural Gradient Ascent introduces additional overhead due to Fisher matrix estimation, but remains relatively efficient due to structured approximations. 

Among our proposed approximation methods, \textbf{ACGN2} incurs further computational cost arising from the intrinsic curvature estimation and HVP-updates; however this overhead remains suppressed due to the absence of outer-product curvature terms.

In contrast, \textbf{ACGN1} is consistently the most computationally expensive, because of the inclusion of both intrinsic and outer-product components which lead to higher variance in curvature estimates and consequently more expensive HVP-based gradient iterations.

As shown in Fig \ref{fig:runtime_benchmarks}, the computational gaps become more evident in higher-dimensional environments (Reacher, Humanoid), where the curvature estimation complexity scales with the policy dimensionality. This highlights a key trade-off: Second-order methods \textbf{improve optimization quality} at the expense of \textbf{increased per-update computation}. \\

$\bullet$ \textbf{ACGN2} provides a favorable balance between computational overhead and curvature quality, offering stable second-order updates at moderate cost due to its intrinsic curvature structure, but may still exhibit sensitivity to noisy curvature estimates in high-dimensional settings. \\

$\bullet$ \textbf{ACGN1} incorporates richer curvature information through both intrinsic and outer-product terms, but this leads to higher computational overhead, increased variance in curvature estimates, and a stronger dependence on damping for stability, limiting its practicality in large-scale or high-dimensional environments.

\textbf{Interpretation:} Although the proposed methods incur higher computational cost per update, the additional computation is compensated by faster convergence in substantially fewer training episodes, thereby reducing the total number of optimization steps required to reach the optimum.

\subsection{Implementation Framework}
All actor-critic methods in our two-timescale architecture were implemented in Python using:
\\
$\bullet$ \textbf{PyTorch}: for policy and critic networks, and automatic differentiation.
\\
$\bullet$ \textbf{Gymnasium}: for standardized RL environments (CartPole, LunarLander, Reacher, Humanoid).
\\
The second-order updates (Natural GA, ACGN1, ACGN2) utilize HVPs to compute update directions as mentioned in \cite{maniyar2024crpn} efficiently without explicit matrix inversion.